\documentclass{article} 
\usepackage{hyperref}
\usepackage{url}
\usepackage{enumitem}
\usepackage{graphicx}
\usepackage[cmex10]{amsmath}
\usepackage{amsthm,amssymb}
\usepackage{algorithm}
\usepackage{algorithmic}
\usepackage{amsmath}
\usepackage{natbib}
\usepackage[final]{nips_2018}

\usepackage{xr}
\newcommand{\extref}[1]{\ref*{#1} of the main paper}

\usepackage{blkarray}
\usepackage{bbm}

\usepackage{nk}

\title{Non-cooperative Multi-agent Systems with Exploring Agents}

\author{Jalal Etesami, Christoph-Nikolas Straehle\\ 
Bosch Center for Artificial Intelligence, Germany}  

\DeclareMathOperator*{\softmax}{softmax}

\begin{document}

\maketitle
\begin{abstract}  
Multi-agent learning is a challenging problem in machine learning that has applications in different domains such as distributed control, robotics, and economics.
We develop a prescriptive model of multi-agent behavior using Markov games.
Since in many multi-agent systems, agents do not necessary select their optimum strategies against other agents (e.g., multi-pedestrian interaction), we focus on models in which the agents play ``exploration but near optimum strategies".
We model such policies using the Boltzmann-Gibbs distribution. This leads to a set of coupled Bellman equations that describes the behavior of the agents. We introduce a st of conditions under which the set of equations admit a unique solution and propose two algorithms that provably provide the solution in finite and infinite time horizon scenarios. 
We also study a practical setting in which the interactions can be described using the occupancy measures and propose a simplified Markov game with less complexity.
Furthermore, we establish the connection between the Markov games with exploration strategies and the principle of maximum causal entropy for multi-agent systems. 
Finally, we evaluate the performance of our algorithms via several well-known games from the literature and some games that are designed based on real world applications. 
\end{abstract}

\section{Introduction}	

A multi-agent system can be defined as a group of autonomous agents that are interacting in a common environment. 
Due to their rich ability of modeling complex dynamics, multi-agent systems are rapidly finding applications in different fields such as autonomous robotics,  telecommunications, distributed control, and economics.
Although the behavior of agents in a multi-agent system can be predefined in advance, it is often necessary that they explore new behaviors to gradually improve their performance.
Another reason that makes the a priori design of a good strategy even more difficult is that many multi-agent systems contain humans as agents. 
In such systems, modeling agents' preferences for selecting their strategies is often complex.

Stochastic games (SGs) have been used for modeling multi-agent systems. However, most of the existing works consider fully cooperative scenarios \cite{wei2016lenient} or settings in which particular communication between the agents is possible. Considering systems such as autonomous cars and their interactions with pedestrians, it is clear that some multi-agent systems are partially cooperative or even competitive, and in many situations, no communication links can be established between the agents. 
More importantly, in these systems, agents select their strategies knowing that the other agents also select their strategies with the same level of awareness.

Another important property of such systems is that the agents do not always select their best-response strategies. 
This is related to Quantal response equilibrium (QRE) which is a smoothed-out best responses, in  the  sense that agents are more likely to select better strategies than worse strategies \cite{mckelvey1995quantal}.
This idea has also its origins in statistical limited dependent variable models such as in economics \cite{mckelvey1998quantal, goeree2002quantal}, psychology, and in biology \cite{palfrey2016quantal}.
As an example, consider the grid game between two players $\{$A, B$\}$ in Figure \ref{game}. Each player can choose to stay at its current position or move to one of its adjacent neighbors. 
Players A and B want to reach their destinations at a and b, respectively and avoid collision. For simplicity, assume the players are aware of each other's goals, i.e., A knows that B wants to go to b and vice versa but they cannot communicate and they only get to play this game once. 
In real world, the pedestrians (players) may face similar situation and they can reach their goals with small chance of collision by selecting a set of near optimal strategies that consider the behavior of the other pedestrians.

This work aims to mimic the logical abilities of human, by building a method that reasons about the anticipated learning of the other agents and select an exploration but near optimum mixed strategy. 
Overall, the proposed model in this work can be categorized as a prescriptive, non-cooperative agenda according to \cite{shoham2007if}.
In our modeling, we take advantage of game theory and reinforcement learning.
\begin{figure}
\centering
\includegraphics[scale=.44]{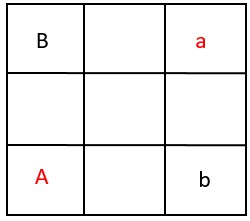}
\caption{A multi-agent system in which the goals of players A and B are to reach a and b, respectively.}\label{game}
\end{figure}
Previous works have also attempted to model similar systems but using different models such as social forces \cite{helbing1995social}, potential fields \cite{alahi2014socially}, flow fields \cite{ali2008floor}, fictitious games \cite{ma2016forecasting}, and others \cite{kretzschmar2014learning,
huang2015approximate,park2016egocentric,
jain2015car,carmel1999exploration}.  
Yet these works are performing either in static environments or in dynamic environments with short-term prediction, and they do not address the complex interactions among the agents. 

\paragraph{Contributions:} 
We propose a game-theoretic model in which all agents select mixed strategies that are distributed according to the Boltzmann distribution. In this case, the strategies of the agents depend on their Q-functions and simultaneously, the Q-functions depend on the strategies. 
\begin{itemize}
\item We introduce a set of coupled Bellman-type equations describing the Q-functions of the agents and show that under some conditions, this set of equations have a unique solution. 
In another words, we introduce a set of assumptions under which there exists a unique QRE for our Markove game.

\item Assuming that the agents are aware of each other's goals\footnote{Goals can be interpreted differently depending on the problem. In this context, we mean the reward (utility) functions.}, we propose two algorithms to obtain the solutions in finite and infinite time horizon settings.

\item We study a practical setting in which the occupancy measures of the agent can capture their interactions and propose a forward-backward algorithm to obtain their behavior with less complexity compared to the general setting.

\item We establish the connection between our model and the principle of maximum causal entropy. This result can be used to develop an algorithm in which each agent can simultaneously infer the goals of the others and its strategy.
\end{itemize}  


\section{Related Works}\label{sec:related}

In this section, we review some related multi-agent reinforcement learning (MARL) algorithms. For a more comprehensive review please see \cite{busoniu2008comprehensive,albrecht2018autonomous}. 
 
Depending on the overall goal of the agents, MARLs can be divided into cooperative or non-cooperative.
Cooperative SGs are the games in which agents have the same reward function and the learning goal is to maximize the common discounted return. The learning algorithms in such SGs are typically derived from a model-free algorithm known as Q-learning  \cite{lauer2000algorithm,
greenwald2003correlated,bowling2003multiagent,hu2003nash}.
Combination of cooperative and competitive Q-learning was developed in \cite{littman2001friend} called friend-or-foe (FOF). 
The convergence of these algorithms is based on several strong assumptions, which may not be realistic in real problems. For instance, either every stage game during learning has a Nash equilibrium or every stage game has a Nash equilibrium that is both beneficiary for the other agents and the learner does not benefit from deviating from this equilibrium \cite{bowling2000convergence}. Such requirement is satisfied in a small class of problems. 
However, our model is not necessarily cooperative and we show that if the norm of the reward functions are bounded, our model has a unique equilibrium.

The work in \cite{guestrin2002coordinated} develops an approach that simplifies the coordination among the agents when the Q-function can be decomposed additively into local functions such that each local function only depends on the actions of a subset of agents. Works in \cite{kok2005non,kok2005using} study conditions under which such decomposition of an optimal joint action can be guaranteed.
Team Q-learning algorithm is another type of learning algorithm for cooperative SGs and it is based on an assumption that the agents have unique optimal joint actions. Therefore, they are able to learn the common Q-function in parallel \cite{littman2001value}. However, this is rarely the case in real world problems.

Agent-tracking algorithms estimate models of the other agents' strategies and response to them accordingly. For example, the joint action learners in \cite{claus1998dynamics} use empirical models of the other agents' strategies. 
i.e., agent $i$'s model of $j$'s strategy is defined as
$$
\pi^i_j(a_j):={C_j^i(a_j)}/{\sum_a C_j^i(a)},
$$ 
where $\pi^i_j(a_j)$ is agent $i$'s empirical model of agent $j$'s strategy and $C_j^i(a_j)$ denotes the number of times agent $i$ observed agent $j$ taking action $a_j$. 
On the other hand, the FMG algorithm in \cite{kapetanakis2002reinforcement} keeps track of only those actions that yielded good rewards in the past. 
Similar to our work, agents in FMQ use Boltzmann action selection. However, it only works for deterministic dynamics. Moreover, FMQ increases the Q-values of only those actions that produced good rewards in the past. This enforces the agent towards coordination.

Algorithms based on fictitious game \cite{ma2016forecasting,conitzer2007awesome}, MetaStrategy algorithm \cite{powers2005new}, and Hyper-Q learning \cite{tesauro2004extending} are other related examples. 
The AWESOME algorithm in \cite{conitzer2007awesome} uses fictitious play that switches from the best response in fictitious play to a pre-computed Nash equilibrium.
The work in \cite{hernandez2017learning} directly models the distribution over opponents. However, addressing the dynamic behavior of the opponent is missing in this work.
The heuristic algorithm in \cite{ma2016forecasting} also uses the fictitious game approach. Analogous to our work, it models the policies of the agents by Boltzmann distribution. However, unlike our algorithm, during the forecast of the agents' policies, the learning algorithm in \cite{ma2016forecasting} is agnostic to the fact that each agent selects its policy also by forecasting its opponents' policies and it misses theoretical analysis.

Authors in \cite{wang2003reinforcement} propose an optimal adaptive learning for team Markov games in which each agent solves a virtual game that is constructed on top of each stage game. Hence, each agent requires empirically estimating a model of the SG, model of the other agents, and estimating an optimal value function for the SG. Under some conditions, it convergences to a coordinated optimal joint action \cite{wunder2010classes}. 
We use a similar approach in the sense that each agent solves a virtual game to find its policy but in non-stationary and non-cooperative setting and without empirical estimation of the SG.

Works that use policy search method in multi-agent setting are the alternative to the Q-learning based algorithms. Generalized IGA \cite{zinkevich2003online} and GIGA-WoLF \cite{bowling2005convergence} are two such algorithms. However, unlike the setting in this work, both are designed for two-agent and two-action games.
Deep neural network has been also used in MARL problems \cite{vinyals2019grandmaster}. 
Most deep-MARL algorithms are also developed for fully cooperative settings \cite{omidshafiei2017deep,foerster2017counterfactual} and emergent communication \cite{foerster2016learning,sukhbaatar2016learning}. The work in \cite{leibo2017multi} considers general sum settings with independent learner and \cite{lowe2017multi} proposes a centralized actor-critic architecture for efficient training in mixed environments. However, they do not reason about the other agents' behaviors. \cite{heinrich2016deep, lanctot2017unified} use the principle of best response algorithms. However, such best responses may not be desired or executed by the agents.

Analogous to the LOLA algorithm \cite{foerster2017learning} and the algorithm based on generative adversarial networks \cite{metz2016unrolled}, our algorithm makes no assumptions about cooperation among the agents and simply assumes that each agent is maximizing its own return. 
However, LOLA is developed by approximating the value function of a two-player game in which the policy of each player is parameterized with a recurrent neural network. The algorithm in \cite{metz2016unrolled} relies on an end-to-end differentiable loss function, and thus does not work in the general RL setting.

\section{Preliminaries}

\paragraph{Markov Decision Process:}
A Markov decision process (MDP) is specified by a tuple $(\mathcal{X},\mathcal{A}, P_0, P,R)$. The set of states is $\mathcal{X}$ that can be continuous or discrete but in this work we consider a discrete state space and a finite set of actions $\mathcal{A}$. 
The initial distribution $P_0$ describes the initial state $x(0)$. The transition probabilities are denoted by $P(x(t+1)|x(t), a(t))$ that is the probability of transitioning to state $x(t+1)$ after selecting action $a(t)$ at state $x(t)$. 
The agent gets $R(x,a)\in \mathbb{R}$ as a reward for selecting action $a$ at state $x$. A policy $\pi(a|x)$ is a conditional distribution that specifies how an agent selects its actions at state $x$.  Stationary policies do not depend on the time step.

The agent's goal in an infinite-time horizon setting is to maximize, at each time-step $k$, the expected discounted return
\begin{equation}\notag
\mathbb{E}\big[\sum_{\tau\geq0}\gamma^\tau R(x(\tau+k),a(\tau+k))\big],
\end{equation}
where $\gamma\in[0,1)$ is the discount factor and the expectation is taken over the probabilistic state transitions.
The Q-function, $Q^\pi : \mathcal{X}\times\mathcal{A}\rightarrow\mathbb{R}$ for a given policy $\pi$ is defined as
\begin{align*}
\small{Q^\pi(x,a)\!:=\!\mathbb{E}[\sum_{\tau\geq0}\!\gamma^\tau R(x(\tau),a(\tau)) | x_0 = x, a_0= a, \pi].}
\end{align*}
The optimal Q-function is defined as $Q^*(x,a)=\max_\pi Q^\pi(x,a)$ and it is characterized by the Bellman optimality equation: 
\begin{align}\label{eq:q_b}
&Q^*(x,a)=R(x,a)+\gamma\sum_{x'}P(x'|x,a)V^*(x'),\\ \notag
&V^*(x')=\max_a Q^*(x',a).
\end{align}
A broad spectrum of single-agent RL algorithms exists,
e.g., model-based methods based on dynamic programming \cite{bertsekas1995dynamic} and model-free methods based on online estimation of value functions.

\paragraph{Markov Game:}
A natural extension of an MDP to multi-agent environments is using Markov games \cite{thuijsman1992optimality}.
 Markov games are a special case of stochastic games (SGs), that are defined by a tuple $(\mathcal{X}_1,...,\mathcal{X}_M,\mathcal{A}_1,...,\mathcal{A}_M,P_0,P,R_1,...,R_M)$, where $M$ is the number of agents, $\mathcal{X}_i$, $\mathcal{A}_i$, and $R_i$ are the state space\footnote{It is common to represent the state of the game with a variable $\textbf{x}$, but without loss of generality, herein, we describe the state of the game using $\vec{x}$.},  action space and the reward function of agent (player) $i$, respectively. 
In this work, we assume that all agents share the same action space $\mathcal{A}$. We denote the state of agent $i$ at time $\tau$ by $x_i(\tau)$, and the state of its opponents by $\vec{x}_{-i}(\tau)$. We also denote the policy of agent $i$ at time $\tau$ by $\pi^\tau_i(a|\vec{x}):\mathcal{X}_1\times...\times\mathcal{X}_M\times\mathcal{A}\rightarrow\mathbb{R}$. Stationary policies are denoted analogously but without time superscript. 
When $R_i=R$, for all $i\in[M]:=\{1,...,M\}$, the agents have the same goal and the SG is cooperative. 

\section{Markov Game with Exploration Strategies}\label{sec:main}
\subsection{Infinite-time Horizon}
As we mentioned, most of the existing works consider scenarios in which the agents play their best strategies against their opponents' best or random strategies. However, in different settings such as human movements in a crowd, the behaviors (strategies) are not necessary optimal but rather close to optimal. 
In this work, we assume that all agents select their mixed strategies such that actions with higher Q-functions\footnote{It represents the quality of an action in a given state.} are selected with higher probabilities (Boltzmann distribution), i.e., \begin{align}\label{policyexp}
\pi_i(a_i|\vec{x})\propto{\exp\left(\beta Q_i(\vec{x},a_i)\right)}, \ \ \text{for}\  i\in[M],
\end{align}
where $0< \beta$ is known as the inverse temperature. 
This particular form of policy allows us to model the behavior of exploring agents (e.g., human) in a common environment. 
The interactions between the agents are encoded into their Q-functions.
We define the Q-function of agent $i\in[M]$ in infinite-time horizon setting as follows,
\begin{align}
&Q_i(\vec{x},a_i)=R_i(\vec{x},a_i)+\gamma\mathbb{E}_{\pi_{-i}}[V_i(\vec{x}')|\vec{x},a_i],  \label{Q_fun}
\end{align}
In the above equations, $\vec{x}$ denotes $(x_i,\vec{x}_{-i})$, and the expectation is taken over the probabilistic state transitions and the strategies of the $i$'s opponents ($-i:=[M]\setminus\{i\}$). 
The first term in \eqref{Q_fun} is the individual reward of agent $i$ that represents the immediate effect of action $a_i$ at state $\vec{x}$. The second term encodes the future effect of selected action $a_i$ considering the behavior of $i$'s opponents.  
In this equation, $V_i$ denotes the value-function,
\begin{equation}\label{eq:va-f}
V_i(\vec{x}')\!:=\!\mathbb{E}_{\pi_i}[Q_i(\vec{x}';a')]\!=\!\!\!\sum_{a'}\pi_i(a'|\vec{x}')Q_i(\vec{x}';a').\!
\end{equation}
Since in \eqref{policyexp}, the agents select their actions based on the state of the game, the joint policy of the agents can be factorized as follows 
\begin{align}\label{ind}
\small{\pi_{[M]}(\vec{a}|\vec{x})=\prod_{j\in[M]}\pi_j(a_j|\vec{x}).}
\end{align}
 It worth noting that \cite{lauer2000algorithm} defines the Q-function by assuming that the opponents of agent $i$ select their best actions, (taking maximum over the actions instead of the expectation in \eqref{Q_fun}). But in this work, analogous to \cite{ma2016forecasting}, we assume that all the agents select their actions according to the Boltzmann distribution.
\begin{remark}
An alternative definition of the value-function is to use $\softmax$\footnote{It is given by $\softmax_x f(x):=\log\sum_x \exp(f(x))$.} instead of the expectation in \eqref{eq:va-f} \cite{ziebart2010modeling,zhou2018infinite}.
The idea of using $\softmax$ is to approximate the $\max$ function in the definition of the Bellman equation \eqref{eq:q_b}. 
However, when the Q-function does not vary too much for different actions, $\softmax(Q)$ will have a bias term of order $\log|\mathcal{A}|$, where $|\mathcal{A}|$ denotes the number of actions\footnote{Consider the following vector of length $n$, $\vec{v}=[z,...,z]$, then $\softmax(\vec{v})=z+\log n$ but $\max(\vec{v})=z$.}. 
On the other hand, the value-function in \eqref{eq:va-f} which is also an approximation of $\max_{a'}Q(x',a')$ does not have this issue.
\end{remark}

Equations (\ref{policyexp})-(\ref{ind}) imply a set of $M$ coupled equations describing the relationships between the Q-functions/policies of the agents. 
More precisely, Equation \eqref{Q_fun}, can be written as 
\begin{align}\notag
&Q_i=R_i + \gamma\mathbb{E}_{Exp\{Q_{-i}\}}[\mathbb{E}_{Exp\{Q_{i}\}}[Q_i]]\\ \label{system1}
&\ \ \ \ :=T_i(Q_{-i},Q_i),\ \ \ \text{for}\ \ i\in[M].
\end{align}
In the above equations, we removed all the arguments only for simpler representation, and $Exp\{Q_i\}$ denotes the policy of $i$th player that is given in \eqref{policyexp}. 
This model describes the behavior (policy selection) of a set of interacting agents that do not always select their best responses but rather their near optimal responses.

This raises two main questions that we will address in this section: does the system of equations in \eqref{system1} admits a unique solution? If so, how can agent $i$ obtain its policy? 

Next result shows that if the reward functions are bounded, $T_i$s are contraction mappings. Therefore, the equations in \eqref{system1} admit a unique solution. The sketch of proof is provided in the Section \ref{sec:proof}.
\begin{theorem}\label{thmma}
Assume that $\max_i||R_i||_\infty\leq\frac{(1-\gamma)^2}{2\gamma M\beta}$. Then, for every $i\in [M]$, $T_i$ is a contraction mapping.  
\end{theorem}
The assumption of Theorem \ref{thmma} is not restrictive as one can ensure it is satisfied by simply scaling the reward functions. However, this may not be the case in some scenarios. Later, in the Section \ref{sec:exp}, we introduce an alternative that relaxes this assumption to incorporate rewards with higher norms.
Since $\{T_i\}$ are contraction mappings, one can find the solution of \eqref{system1} using a value iteration algorithm. Here, we present MGE-I in Algorithm \ref{alg:1} that can $\epsilon$-approximate the solution of \eqref{system1}.
 \begin{algorithm}
\caption{MGE-I}\label{alg:1} 
\begin{algorithmic}
\STATE $\{R_{j}\}, \epsilon$
\STATE \textbf{Initialize:} $\vec{Q}^0, \vec{Q}^1$, $s=0$
\WHILE{$\max_j||Q_j^{s+1}-Q_j^s||\geq\epsilon$}
\FOR {$j\in-i$}
\STATE $Q_j^{s+1}\leftarrow T_j(Q_{-j}^{s},Q_j^{s})$
\ENDFOR
\STATE $Q_i^{s+1}\leftarrow T_i(Q_{-i}^{s+1},Q_i^{s})$
\STATE $s\leftarrow s+1$
\ENDWHILE
\end{algorithmic}
\end{algorithm}

\subsection{Finite-time horizon}
In the finite-horizon setting, the problem is slightly different. This is because of non-stationary policies and no discount factor $\gamma$ in the definition of returns. 
More precisely, in our model, the Q-function of the $i$th agent at time $\tau\in[0,T]$ is defined by
\begin{align}\label{finite-t}
&Q^{\tau}_i( \vec{x},{a}_i):=R_i( \vec{x}, {a}_i) + \mathbb{E}_{\pi^\tau_{-i}}[ V^{{\tau}+1}_i( \vec{x}')|\vec{x}, {a}_i],\\ \notag
&V^{\tau+1}_i( \vec{x}')=\mathbb{E}_{\pi_i^{\tau+1}}[Q^{\tau+1}_i(\vec{x}', {a}')],
\end{align}
with the boundary condition $V_i^{T}(\vec{x})=R_{i,F}(\vec{x})$ that is the final reward of agent $i$.
Analogous to the infinite-time horizon setup, the expectation in \eqref{finite-t} depends on the policy of the other agents, which we assume that are all distributed according to the Boltzmann distribution.
Thus, Equation (\ref{finite-t}) can be written as
\begin{align}\label{er0}
Q^{\tau}_i=U_i(Q^{\tau}_{-i},V^{\tau+1}_i),\ \ \text{for}\ i\in[M],
\end{align}
where $U_i$ denotes the right hand side of Equation (\ref{finite-t}). 
Similar to the result of infinite-time horizon section, we have the following result.

\begin{theorem}\label{thmm2}
If $\max_i\{||R_i||_\infty,||R_{i,F}||_\infty\}\leq 1/2\beta(M-1)(1+T)$, then $U_i$ is a contraction mapping with respect to its first argument.
\end{theorem}
This result shows that the system of equations in \eqref{er0} admits a unique set of solution $\{\vec{Q}^0,...,\vec{Q}^{T}\}$ and Algorithm \ref{alg:2} is able to $\epsilon$-approximate the solutions. 
\begin{algorithm}
\caption{MGE-F}\label{alg:2}
\begin{algorithmic}[1]
\STATE $\{R_j,R_{j,F}\}, T, \epsilon$
\FOR {$j=1,...,M$}
\STATE $V_j^{T}\leftarrow R_{j,F}$
\ENDFOR 
\STATE \textbf{Initialize:} $\vec{Q}^0, \vec{Q}^1$
\FOR{$\kappa=T-1,...,0$}
\STATE $s=0$
\WHILE{$\max_j||Q_j^{s+1}-Q_j^s||\geq\epsilon$}
\FOR {$i=1,...,M$}
\STATE $Q_i^{s+1}\leftarrow U_i(Q_{-i}^{s},V_i^{\kappa+1})$
\ENDFOR
\STATE $s\leftarrow s+1$
\ENDWHILE
\STATE {$\vec{Q}^0\leftarrow \vec{Q}^{s}$}
\FOR {$j=1,...,M$}
\STATE{$\widehat{Q}_j^{\kappa}\leftarrow U_j(Q^{0}_{-j},V^{\kappa+1}_j)$}
\STATE {$V^{\kappa}_j\leftarrow\mathbb{E}_{\pi_j^{\kappa}}[\widehat{Q}_j^{\kappa}]$}
\ENDFOR
\ENDFOR
\end{algorithmic}
\end{algorithm}

\section{Special Setting}
In many applications, the interactions among the agents can be summarized to not-colliding or colliding (not merging/merging into a same state) in a common environment. Example of a not colliding scenario is a set of self-driving cars in a highway and an example of colliding scenario is two robots exchanging their loads in a warehouse.
The need for such model is to reduce the complexity of the inference algorithms in the previous section.
In this section, we introduce a simplified Markov game that can model such scenarios. 
This game has lower complexity and requires less memory compared to the games in Section \ref{sec:main}. For the rest of this section, we assume that all agents have the same state space $\mathcal{X}$. 

The main idea of this model is to assume that the reward function of agent $i$ in \eqref{finite-t} can be factorized into two terms: the individual goal of the agent that unlike $R_i(\vec{x},a_i)$ in \eqref{finite-t} depends only on $(x_i, a_i)$ and another term that encodes the interactions between $i$ and its opponents. The latter term should depends only on the existence probability of $i$'s opponents at state $x_i$. 
Therefore, we define the \textit{occupancy measure} $O_j^\tau(x_i):\mathcal{X}\rightarrow[0,1]$ to denote the likelihood of agent $j$ being at state $x_i$ at time $\tau$.

With this assumption, we redefine the Q-function of agent $i$ at time $\tau$, which now, it only depends on $i$'s state and action.
 More precisely, we define
\begin{align}\label{Q_fun_a}
&\widetilde{Q}^\tau_i(x_i,a_i):=R_i({x}_i,a_i)\!+\!\Psi\big(O_{-i}^{\tau}(x_i)\big)\!+\!\mathbb{E}[\widetilde{V}^{\tau+1}_i(x)|{x}_i,a_i],  
\end{align}
where $O_{-i}^\tau(x_i):=\{O_1^\tau(x_i),...,O_M^\tau(x_i)\}\setminus\{O_i^\tau(x_i)\}$,
and $\Psi(\cdot)$ is a functional\footnote{As an example, in the self-driving car scenario, a possible choice for $\Psi(O_{-i}^\tau(x_i))$ is $-\sum_{j\in-i}\mu_jO_j^\tau(x_i)$, where $\mu_j>0$.}. 
The expectation is taken over the probabilistic state transitions.  
The value-function is defined similar to \eqref{finite-t}, 
\begin{equation}\label{eq:va-f_a}
\widetilde{V}^{\tau+1}_i(x):=\mathbb{E}_{\widetilde{\pi}^{\tau+1}_i}[\widetilde{Q}^{\tau+1}_i(x;a)],
\end{equation}
and $\widetilde{V}_i^T(x)=R_{i,F}(x)$, where $R_{i,F}(x)$ denotes the final reward of agent $i$ at state $x$. The policy $\widetilde{\pi}^{\tau}_i=Exp\{\widetilde{Q}_i^{\tau}\}$ is defined similar to \eqref{policyexp}.
The above equations lead to  
\begin{align}\notag
&\widetilde{Q}^\tau_i= R_i+\!\Psi(O^\tau_{-i})\!+\mathbb{E}_{Exp\{\widetilde{Q}_i^{\tau+1}\}}[\widetilde{Q}_i^{\tau+1}]\\  \label{backw}
&:=B_i(O_{-i}^\tau, \widetilde{Q}_i^{\tau+1}), \ \text{for}\ \ i\in[M],
\end{align}
This equation describes the dependency  between the Q-function and the occupancy measures.
On the other hand, the occupancy measure can be written recursively as follows
\begin{align}\notag
&O_i^{\tau+1}(x)=\!\!\sum_{x',a} \widetilde{\pi}_i^{\tau}(a|x')P(x|x',a)O_i^{\tau}(x')=\! \mathbb{E}_{Exp\{\widetilde{Q}_i^{\tau}\}}[O_i^{\tau}]\\ 
\label{forw}
&:=G_i(O_{i}^{\tau}, \widetilde{Q}_i^{\tau}),\ \ \text{for}\  i\in[M],
\end{align}
\begin{figure*}
\centering
\includegraphics[scale=.4]{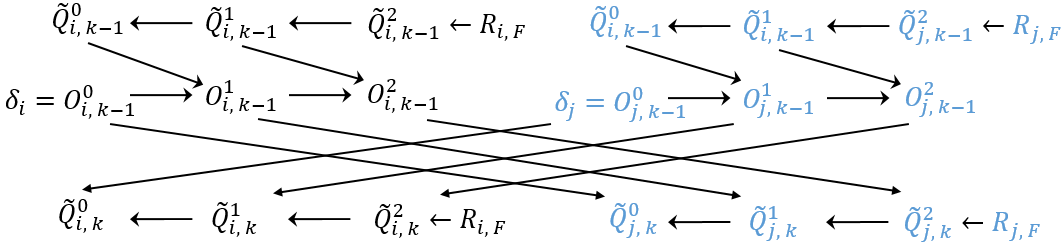}
\hspace{.4cm}
\includegraphics[scale=.32]{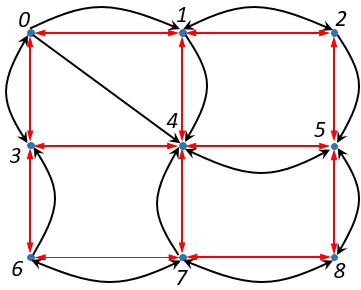}
\caption{Left: Graphical model demonstrating the dependencies between variables in algorithm \ref{alg:fb} with two agents $\{i, j\}$ and $T=2$. Here, $k$ denotes the $k$th iteration of the algorithm. Right: State space and possible actions of the pursuit game. Red and black arrows including self-loop for every node indicate the possible actions of the hunter and the prey, respectively.}\label{fig:fb}
\end{figure*}
and the boundary condition $O^0_i=\delta_i$, where $\delta_i$ denotes the delta function that is zero everywhere except at the current state of agent $i$. 
By assuming that agents are aware of each others goals (reward functions $\{R_j, R_{j,F}\}_{j\in[M]}$ are known to all the agents), we propose a forward-backward algorithm in \ref{alg:fb} that outputs the Q-functions of the agents.
In the forward-pass of Algorithm \ref{alg:fb}, the occupancy measures $\{O^\tau_i\}$ are updated using \eqref{forw} and afterward, they are used in the backward-pass to update the Q-functions via \eqref{backw}.
Figure \ref{fig:fb} demonstrates the dependencies between the variables of a two-players system with $T=2$.
\begin{algorithm}[H]
\caption{MGE-FB}\label{alg:fb}
\begin{algorithmic}[1]
\STATE $\{R_j,R_{j,F}\}, T, K$
\STATE \textbf{Initialize:} $\{\widetilde{Q}_j^0\},...,\{\widetilde{Q}_j^T\}$
\FOR {$j=1,...,M$}
\STATE $O_j^{0}\leftarrow \delta_j, \ \ \widetilde{Q}_j^{T+1}\leftarrow R_{j,F}$
\ENDFOR
\FOR{$k=1,...,K$}
\STATE {\textit{Forward-pass:}}
\FOR {$\tau=1,...,T$ and $j=1,...,M$}
\STATE $O_{j}^{\tau}\leftarrow G_j(O_{j}^{\tau-1},\widetilde{Q}_{j}^{\tau-1})$
\ENDFOR
\STATE{\textit{Backward-pass:}}
\FOR {$\tau=T-1,...,0$ and $i=1,...,M$}
\STATE $\widetilde{Q}_{i}^{\tau}\leftarrow B_i(O_{-i}^{\tau},\widetilde{Q}_{i}^{\tau+1})$
\ENDFOR
\ENDFOR
\end{algorithmic}
\end{algorithm}
Next result introduces a condition under which Algorithm \ref{alg:fb} converges to a set of unique Q-functions as the number of iterations $K$ tends to infinity. To establish this result, we assume that there is a constant\footnote{ For the choice of $\Psi(O_{-i}^\tau(x_i))=-\sum_{j\in-i}\mu_jO_j^\tau(x_i)$ and when the state space $\mathcal{X}$ is discrete, we have $L=\sum_j\mu_j$.} $L$ such that for any pairs of occupancy measures $O, \overline{O}:\mathcal{X}\rightarrow[0,1]$ and $\forall\ i\in[M]$,
\begin{align*}
& ||\Psi(O_{-i})-\Psi(\overline{O}_{-i})||_\infty \leq L \max_{j\in-i}||O_j-\overline{O}_j||_1.
\end{align*}
Furthermore, we assume that there exist constants $\omega$ and $\varphi$ such that $\max_i\{||R_i||_\infty,||R_{i,F}||_\infty\}\leq \omega$ and $||\Psi||_\infty\leq \varphi$ and let $\xi:=(T+1)(\omega+\varphi)$.
\begin{theorem}\label{them_conv}
Under the above assumptions, Algorithm \ref{alg:fb} converges as $K$ tends to infinity when $2LT\leq \xi\exp(-\beta(T+1)\xi)$. 
\end{theorem}

\section{Multi-agent Maximum Causal Entropy}\label{sec:learning-individual}
In the above formulation of the Markov games, it is assumed that the agents are aware of all the reward functions. However, this may not be the case in some problems. 
In order to relax this assumption, we establish the connection between Markov games with exploration strategies and the principle of maximum causal entropy (MCE). Then, we show how this connection can be used to relax the aforementioned assumption.

The principle of MCE prescribes a policy by maximizing the entropy of a sequence of actions causally conditioned on sequentially revealed side information  \cite{ziebart2010modeling}. 
The original formulation of MCE is for single-agent setting but the multi-agent extension of MCE is also introduced in \cite{ziebart2011maximum}.
The problem of MCE in multi-agent setting is as follows,
\begin{equation}\label{irl_max}
\begin{aligned}
&\max_{\{\pi^\tau_i\}} \small{H(\vec{a}||\vec{x}):=\small{-\mathbb{E}_{\vec{a},\vec{x}}\Big[\sum_{\tau\leq T}\log \pi^\tau(\vec{a}(t)|\vec{x}(t))\Big]},}\\
&\text{s.t.}\ \  \mathbb{E}_{a_i,\vec{x}}[F_i(\vec{x},a_i)]=\widehat{\mathbb{E}}_{a_i,\vec{x}}[F_i(\vec{x},a_i)],\ \ \forall i, \\ 
&\ \ \sum_{a_i(\tau)\in \mathcal{A}}\!\!\pi^\tau_i(a_i(\tau)|\vec{x}(\tau))\!=\!1,\ \ \ \ \forall\ \tau, i, \vec{x}(\tau), \\
&\ \ \ \ \pi^\tau_i(a_i(\tau)|\vec{x}(\tau))\geq0,\ \ \ \ \forall\ \tau, i, a_i(\tau), \vec{x}(\tau),
\end{aligned}
\end{equation}
The first constraint is to ensure that for any agent $i$, the expectation of the feature function $F_i=(F_{i,1},...,F_{i,N_i})$ matches its empirical mean $\widehat{\mathbb{E}}_{a_i,\vec{x}}[F_i(\vec{x},a_i)]$.
The feature functions are revealed as side information to all agents.
Next results describe the solution of \eqref{irl_max}.

\begin{theorem}\label{thm:G}
Solution of \eqref{irl_max} is recursively given by
\begin{align*}
&\pi_{i}^\tau\!=\!\frac{1}{Z_i(\tau)}\exp\Big(W^\tau_i(\vec{x}(\tau),a_i(\tau))\Big),\ \\ 
&W^\tau_i(\vec{x}(\tau)\!,a_i(\tau))=<\theta_i, F_i(\vec{x}(\tau),a_i(\tau))>+\mathbb{E}_{\pi_{-i}^\tau}[\log Z_i(\tau\!+\!1)],\\
&\log Z_i(\tau)=\softmax_{a'\in\mathcal{A}}W^\tau_i(\vec{x}(\tau),a'),\\
& \log Z_i(T)=\softmax_{a\in\mathcal{A}} <\theta_i, F_i(\vec{x}(T),a)>,
\end{align*}
where $\pi_{-i}^\tau=\prod_{j\in-i}\pi_j^\tau$.
The boundary condition is $Z_i(T+1)=1$ for all the agents.
\end{theorem}

This result resembles the setting of our Markov games with exploration strategies; $W^\tau_i(\vec{x}(\tau),a_i(\tau))$ plays the role of Q-function in \eqref{finite-t} but with the difference that $\softmax$ is been used to define the value-function instead of the expectation in \eqref{finite-t}. $<\theta_i, F_i(\vec{x}(\tau)\!,a_i(\tau))>$ is the reward function, and $\pi^\tau_i$ is the policy and it is distributed according to the Boltzmann distribution with $\beta=1$.

An important consequence of this result is to be able to develop a gradient-based algorithm similar to the MCE-IRL \cite{ziebart2008maximum,bloem2014infinite} or an online inverse reinforcement learning (IRL) algorithm similar to \cite{rhinehart2017first} that can simultaneously infer the rewards and the policies. Such algorithm requires the gradient function that is given below.
\begin{theorem}\label{thm:grad}
The gradient of the dual problem with respect to $\theta_i$ is given by
$
\widehat{\mathbb{E}}[F_i(\vec{x},a_i)]-{\mathbb{E}}[F_i(\vec{x},a_i)].
$
\end{theorem}
\begin{algorithm}
\caption{Online MMCE-IRL for agent $i$}\label{alg:3}
\begin{algorithmic}[1]
\STATE {initialize: $\{\theta_j\}$}
\FOR {$j\in-i$}
\STATE $\widetilde{F}^{(\tau)}_j\leftarrow \sum_{(\vec{x},a_j)\in H^t} F_{j}(\vec{x},a_j)$
\STATE $R_j^{(\tau)} \leftarrow <\theta_j, F_{j}>$
\ENDFOR 
\STATE $\vec{\pi}\leftarrow \text{MGE}(\{R_j^{(\tau)}: j\in-i\},R_i)$
\FOR {$j\in[M]$}
\STATE $\bar{F}_j\leftarrow\mathbb{E}_{\vec{\pi}}[F_j]$
\STATE $\small{\theta_j\leftarrow\text{Project}_{||\theta_j||\leq B}\big(\theta_j-\varrho(\widetilde{F}^{(\tau)}_j-\bar{F}_j)\big)}$
\ENDFOR
\end{algorithmic}
\end{algorithm}
Algorithm \ref{alg:3} summarizes the steps of an online inverse reinforcement learning (IRL) approach for agent $i$ in which this agent infers the individual rewards of its opponent by observing their behavior and knowing its own reward function. Similar to \cite{rhinehart2017first}, we include a projection step to ensure that the reward functions are bounded. In Algorithm \ref{alg:fb}, $H^t$ denotes the trajectories up to time $t$, i.e., pairs of states and actions and $\varrho$ is the step size of the gradient descent.

\section{Experimental Results}\label{sec:exp}
This section summarizes different settings where we evaluated the performance of our algorithms.

\textbf{Pursuit Game:}
The pursuit problem has been widely examined in multi-agent systems \cite{weinberg2004best}. 
We used different variants of this problem in our experiments: 
First, we consider two agents; one hunter (h) and one prey (p). 
Figure \ref{fig:fb} illustrates the state space (nodes of the graph) and the possible actions of the agents. 
The red and black arrows indicate the possible actions for the hunter and the prey, respectively. 
Each agent can choose to stay at its current position or follow one of its corresponding arrows.
This game is deterministic, i.e., specifying the states and the actions, the next state of the agents are known with probability one.
The agents can observe the current position of their opponent and they move simultaneously.
The goal of the predator is to hunt the prey during a time horizon of length $T=22$. 

This is a Markov game in which $-R_p(x_h,x_p,a_p)=R_h(x_h,x_p,a_h)$ $=0.4$, when $x_h=x_p\in\{0,1,...,8\}$ and zero, otherwise. 
A similar but simpler variant of this game was studied by \cite{weinberg2004best} in which both agents play on the same grid and the prey moves left or up at random.
The output of MGE-F is a set of time-dependent mixed strategies but we assumed that the agents only executed the actions with the maximum  probabilities. 
In this case, the average score of the hunter and its average number of hunts using MGE-F during 10 games with random initial states were $1.16$ and 2.9, respectively. We compared the performance of our algorithm with NSCP in \cite{weinberg2004best} for which both players used the NSCP to learn their policies. The results of NSCP were 0.84 and 2.1.

In the second variant, we consider three agents; two hunters $\{h_1,h_2\}$ and one prey $\{p\}$.
They move on the same graph as in Figure \ref{fig:fb} but unlike the previous setting, they can either stay at their current position or move according to the red arrows. The hunters have only $T=3$ steps to catch the prey without colliding with each other. 
We used the following reward for $h_1$,
\begin{small}\begin{align}\notag
&R_{h_1}(x_{h_1},x_{h_2},x_p,a_{h_1})=\begin{cases} 
      0 & x_{h_1}\not\in \{x_{h_2}, x_p\} \\
      -15/4 & x_{h_1}=x_{h_2}\neq x_p \\
      -10/4 & x_{h_1}=x_{h_2}= x_p \\
      +5/4 & x_{h_1}=x_p\neq x_{h_2}.
   \end{cases}
\end{align}\end{small}
The reward of the second hunter was defined similarly, and the prey's reward was given by $R_{p}(x_{h_1},x_{h_2},x_p,a_{h_1})=0$, when $x_{p}\not\in \{x_{h_1},x_{h_2}\}$ and -1/8, otherwise.

To relax the assumption of Theorem \ref{thmm2}, we modified the update rule of MGE-F. 
More precisely, we reduced the exploration property of the agents by adding a portion of the previously estimated Q-function to the new Q-function, i.e., instead of line 10 of MGE-F, we used
\begin{align}\label{mod}
Q_i^{s+1}\leftarrow \alpha U_i(Q_{-i}^{s},V_i^{t+\kappa+1}) + (1-\alpha)Q_{i}^{s},
\end{align}
where $\alpha\in[0,1]$. Note that $\alpha=1$ leads to the same update rule as in MGE-F.
This modification will not disturb the convergence of the algorithm as long as $\gamma_{\alpha b}+(1-\alpha)<1$, where $\gamma_{\alpha b}$ denotes the contraction coefficient of $U_i$ when $\alpha\max_j\{||R_j||_\infty\}\leq b$. 
Thus, by selecting proper $\alpha$, the convergence can be guaranteed for reward functions that have norms greater than the bound in Theorem \ref{thmm2}. 
The cost for this relaxation is losing the convergence speed.
Figure \ref{alpha_eff} shows the convergence error for the second variant of the pursuit game when $\alpha\in\{0.05,0.2,0.4,0.6\}$.
All different values of $\alpha$ led to the same set of policies.  In all the experiments in this section, we used $\beta=1$.

\begin{figure}
\centering
\includegraphics[scale=.5]{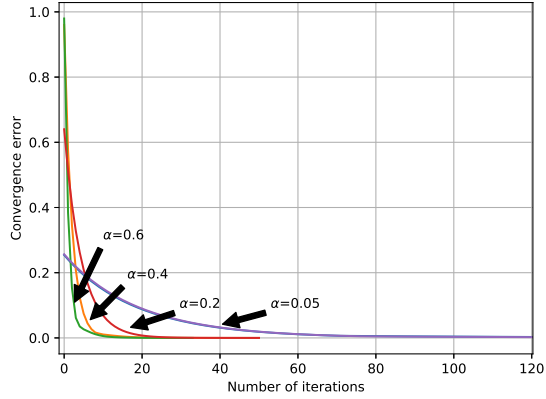}
\caption{Effect of $\alpha$ on the convergence speed of the modified updating rule in (\ref{mod}) for the second variant of the pursuit game. } \label{alpha_eff} 
\end{figure}
Table in Figure \ref{t2} presents the learned policies for the first moves of $h_1$ and $p$, when the initial states are $I_1:=\{h_1\in0, h_2\in8, p\in4\}$\footnote{$h_1$ is at node 0, $h_2$ is at node 8 and the prey is at node 4.} and $I_2:=\{h_1\in0, h_2\in8, p\in2\}$.
 
\begin{figure}
\centering
\includegraphics[scale=.33]{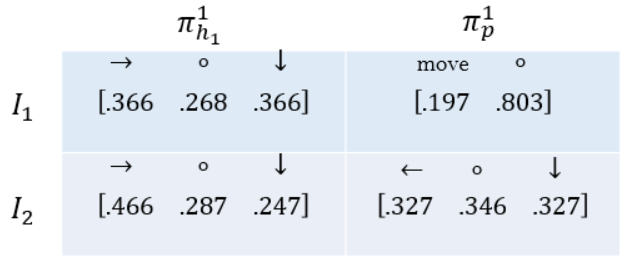}
\caption{The learned mixed policies for the first move of $h_1$ and $p$ in the pursuit games with three players. }\label{t2} 
\end{figure}
Another interesting variant of this game is when the initial states of the players are $\{h_1\in1, h_2\in5, p\in2\}$ and $T=1$, i.e., each player can only execute one action. 
As we expected, the output of MGE-F was that the best action\footnote{Action with the maximum probability.} for each player is not to move. 
This behavior can be explained using the predictive ability of the players. 

\textbf{Rabbit-Hole Game: }
This is also a pursuit game in which a fox chasing a rabbit over a grid, Figure \ref{game2}. There is a prize +$0.3$ for the rabbit in a small hole. Both rabbit and fox can enter the hole. Fox will gain +2 points, when it catches the rabbit and the rabbit loses 2 points. 
Each agent has only $12$ moves. 
We modeled this game using a finite-time horizon SG with $T=12$, and learned the policies using MGE-F. 
Interestingly, the behavior of the learned policies was that with high probability, the rabbit entered the hole when it knew that the fox is far enough from the entrance and the fox moved such that its distance to the rabbit and the entrance is minimized.   
The average scores of the rabbit and the fox after playing 10 rounds with random initial states were 0.06 and 0, respectively.

\textbf{Grid Games:}
We also studied the behavior of our learning algorithm in two different grid games all of which are two-player games: (I) a stochastic game, Figure \ref{game} and (II) a cooperative game that is the stochastic version of the Battle of Sexes, Figure \ref{game2}. 
In (I), agents are rewarded +$30$, when they reached their goals and punished -$1$, when they collided.
In (II), both agents wish to reach G without collision, but if they try to go over the barrier (indicated by curves in Figure \ref{game2}), they may fail with probability $0.5$. The reward at G was +2 and the collision cost was losing one point.

We compared our algorithm with three algorithms in \cite{greenwald2003correlated}. The selected algorithms (Q-learning, uCE-Q, lCE-Q)  were trained by repeatedly playing the games. As it is discussed in \cite{greenwald2003correlated}, they all converge to a symmetric Nash equilibrium in (I) and asymmetric Nash equilibrium in (II).  
It is important to mention that the main difference between MGE and the algorithms in \cite{greenwald2003correlated} is that they learn the best-response policies for the agents but MGE learns the exploration policies.   
Table in Figure \ref{game2} shows the average scores of different learning algorithms after playing the games 1000 times. 
\begin{figure}
\centering
\vspace{.7cm}
\includegraphics[scale=.34]{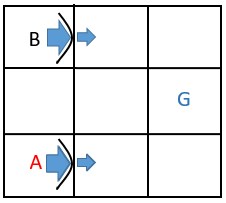}
\hspace{.3cm}
\includegraphics[scale=.38,angle =90]{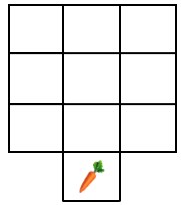}
\hspace{.6cm}
\includegraphics[scale=.58,angle =0]{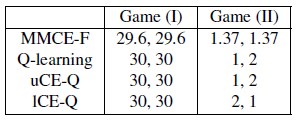}
 \caption{Upper left: Initial states of the grid game (II). G denotes the goal of both agents. They can successfully pass the barriers with probability 0.5. Upper right: The state space of the rabbit-hole game.
Table: Average scores of the players in grid games.}\label{game2}
\end{figure}

\textbf{Driving Scene:}
In this experiment, we simulated a driving scene in which 4 agents (three vehicles and a pedestrian) interact at a road junction. The initial positions and the corresponding goals of the agents are illustrated in Figure \ref{game3}.  In this experiment, the states and the actions are discrete, i.e., agents can choose to stay or move to one of their neighboring cells. 
This complex multi-agent scenario is quite interesting for self-driving cars application.
We used Algorithm \ref{alg:fb} (MGE-FB) with $\Psi(O_{-i}^\tau(x_i))=-\sum_{j\in-i}\mu_jO_j^\tau(x_i)$.
Actions with the maximum probabilities are: all cars stop before the zebra crossing while the pedestrian walks toward its destination. 
Then, car 2 drives towards its goal. Cars 1 and 3 are the next agents that drive to their goals, respectively.  All the agents select their shortest routes. These behaviors match human driving.
\begin{figure}
\centering
\includegraphics[scale=.39]{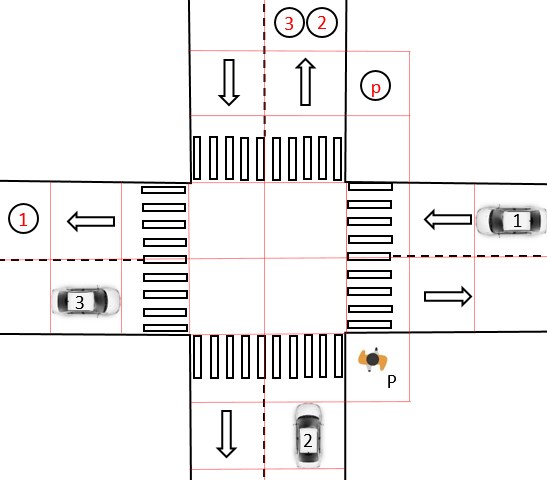}
 \caption{Initial state of the driving scene. Goals are indicated by circles, e.g., car 1 wants to go to circle 1. }\label{game3}
\end{figure}

\section{Discussion}
We developed several multi-agent algorithms in which the agents use exploration policies while considering their opponent's strategies. We showed the convergence of our algorithms and also evaluated their learning behavior using several game-theoretic examples. 
A possible future work is extending the presented results to continuous state-action spaces. This is of great importance for several applications such as robotic and self-driving cars.

\section*{Supplementary Materials}

\begin{theorem}\label{thmma}
Assume that $\max_i||R_i||_\infty\leq\frac{(1-\gamma)^2}{2M\gamma\beta}$. Then, for every $i\in [M]$, $T_i$ is a contraction mapping.  
\end{theorem}

\begin{theorem}\label{thmm2}
If for every $i$, $\max\{||R_i||_\infty,||R_{i,F}||_\infty\}\leq 1/(2\beta(M-1)(1+T))$, then the system of equation in (6) admits a unique solution.
\end{theorem}

\begin{theorem}\label{them_conv}
Suppose that . Then, Algorithm 3 converges to a set of unique Q-functions as the number of iterations $K$ tends to infinity.
\end{theorem}

\begin{theorem}\label{thm:G}
Solution of the multi-agent MCE (7) is recursively given by
\begin{align*}
&\pi_{i}^\tau\!=\!\frac{1}{Z_i(\tau)}\exp\!\Big(W^\tau_i(\vec{x}(\tau),a_i(\tau))\Big),\\
&W^\tau_i(\vec{x}(\tau)\!\!,a_i(\tau))\!=\!\theta_i^T F_i(\vec{x}(\tau)\!,a_i(\tau))\!+\!\mathbb{E}_{\pi_{-i}^\tau}\![\log Z_i(\tau\!+\!1)],\\
&\log Z_i(\tau)\!=\!\!\softmax_{a'}W^\tau_i(\vec{x}(\tau),a'),\\
&\log Z_i(T)=\softmax_{a\in\mathcal{A}} \theta_i^TF_i(\vec{x}(T),a),
\end{align*}
where $\pi_{-i}^t=\prod_{j\in-i}\pi_j^t$.
The boundary condition is $Z_i(T+1)=1$ for all the agents.
\end{theorem}

\begin{theorem}\label{thm:gra}
The gradient of the dual with respect to $\theta_i$ is given by
\begin{align*}
\tilde{\mathbb{E}}[F_i(\vec{x},a_i)]-{\mathbb{E}}[F_i(\vec{x},a_i)].
\end{align*}
\end{theorem}


In order to establish the above results, we require the following results. For simplicity, in the remaining, we use $||\cdot||$ to denote $||\cdot||_\infty$.

\begin{lemma}\label{lemma2}
Let $X=[x_1,...,x_n]\in\mathbb{R}^n$. There exists $0<\alpha<1$, such that
\begin{align}\label{eqlemma}
\frac{1}{2}\sum_{j,i}\lambda_i\lambda_j|e^{x_i}-e^{x_j}|\leq \alpha||X||\left(\sum_k \lambda_k e^{x_k}\right),
\end{align}
where $0<\lambda_i<1$ and $\sum_k \lambda_k=1$.
\end{lemma}
\begin{proof}

We prove this by induction on $n$. First step is when $n=2$. In this case, the left hand side of (\ref{eqlemma}) becomes $
\lambda(1-\lambda)|e^{x_2}-e^{x_1}|$. Without loss of generality, assume $x_2\geq x_1$. We will show that there exists an $\alpha$ such that the function below is positive for $x_2-x_1=x\geq0$,
\begin{align*}
g_\alpha(x):=2\alpha \max\{|x_1+x|,|x_1|\}\left(\lambda e^{x}+(1-\lambda)\right)-2\lambda(1-\lambda)(e^{x}-1).
\end{align*}
If 
Since $2\max\{|x_1+x|,|x_1|\}\geq |x|$, we can show that $g_\alpha(x)$ is positive, if $h_\alpha(x)$ is positive,
\begin{align*}
h_\alpha(x):=\alpha x\left(\lambda e^{x}+(1-\lambda)\right)-2\lambda(1-\lambda)(e^{x}-1).
\end{align*}
This is straightforward, because for $1>\alpha>(1-\lambda)\max\{1, 2\lambda\}$, we have $h_\alpha(0)=0$, $\partial h_\alpha(0)/\partial x>0$, and $\partial^2 h_\alpha(x)/\partial x^2\geq0$.

Induction hypothesis; for any vector $Z\in\mathbb{R}^n$, there exists an $\alpha$ such that (\ref{eqlemma}) holds. Let $X\in\mathbb{R}^{n+1}$ such that $|x_{n+1}|=||X||$, and define $x_{n+1}:=x^*$. By induction hypothesis, we have
\begin{align*}
\sum_{j,i\leq n}\lambda_i\lambda_j|e^{y_i}-e^{y_j}|\leq& 2(1-\lambda)\alpha^*||Y||\sum_{k\leq n} \lambda_k e^{y_k},
\end{align*}
where $Y:=[x_{n},...,x_1]\in\mathbb{R}^n$, in which the zero entry is removed. Therefore, to prove the $n+1$st step, we need to show that
\begin{align}\label{ine}
&2(1-\lambda)\alpha^*||Y||\sum_{k\leq n} \lambda_k e^{y_k}+2\lambda\sum_{i\leq n}\lambda_i|e^{y_i}-e^{x^*}|\leq 2\alpha||X||(\sum_{k\leq n} \lambda_k e^{y_k}+\lambda e^*).
\end{align}
Let $1>\alpha\geq\alpha^*$, and note that $||Y||\leq||X||=|x^*|$. Therefore, we can prove (\ref{ine}) by showing 
\begin{align}\label{ine2}
&0\leq \alpha |x^*|(\sum_{k\leq n} \tilde{\lambda}_k e^{y_k}+\frac{e^{x^*}}{1-\lambda})-\sum_{i\leq n}\tilde{\lambda}_i|e^{y_i}-e^{x^*}|.
\end{align}
where $\tilde{\lambda}_i:=\frac{\lambda_k}{1-\lambda}$. Note that $\sum_k\tilde{\lambda}_k=1$. 
Without loss of generality, let $y_1\leq...\leq y_n$. Since $|x^*|\geq ||Y||$, there are only two possible scenarios: I) $x^*\geq0$, then $-x^*\leq y_1\leq...\leq y_n\leq x^*$ or II) $x^*<0$ in which, we have $x^*\leq y_1\leq...\leq y_n\leq -x^*$.

I) In this case, the right hand side of (\ref{ine2}) can be bounded by
\begin{align*}
&(\alpha x^*+1)\sum_{k\leq n} \tilde{\lambda}_k e^{y_k}+\alpha x^*\frac{e^{x^*}}{1-\lambda}-e^{x^*}\geq (\alpha x^*+1)e^{-x^*}+\alpha x^*\frac{e^{x^*}}{1-\lambda}-e^{x^*}:=J_\alpha(x^*).
\end{align*}
We show that there exists $0<\alpha<1$ such that for all $x^*>0$, $e^{x^*}J_\alpha(x^*)\geq0$. 
Because $J_\alpha(0)=0$, $\partial e^{x^*}J_\alpha(x^*)/\partial x^*$ is $\alpha+\alpha/(1-\lambda)-2$ for $x^*=0$, and 
$$
\frac{\partial^2 e^{x^*}J_\alpha(x^*)}{\partial (x^*)^2}\Big|_{x^*=0}=\left(4\frac{\alpha}{1-\lambda}-4+\frac{4\alpha x^*}{1-\lambda}\right)e^{2x^*},
$$
selecting $1>\alpha\geq\max\{1-\lambda,\frac{2-2\lambda}{2-\lambda}\}$, we have $J_\alpha(x^*)\geq0$.

II) Suppose $x^*=-x<0$. In this case, the right hand side of (\ref{ine2}) becomes
\begin{align*}
&(\alpha x-1)\sum_{k\leq n} \tilde{\lambda}_k e^{y_k}+\alpha x^*\frac{e^{-x}}{1-\lambda}+e^{-x}.
\end{align*}
Clearly, the above equation is positive if $\alpha x-1>0$, otherwise, it is bounded by the following function 
$$
h_\alpha(x):=(\alpha x-1)e^{x}+\frac{\alpha x}{1-\lambda}e^{-x}+e^{-x}.
$$
We show there exists an $\alpha<1$ such that $r_\alpha(x):=h_\alpha(x)e^x$ is positive for all $\alpha x<1$. This is true because of the following facts 
\begin{align*}
&r_\alpha(0)=0,\ \ \ \ \frac{\partial r_\alpha(x)}{\partial x}\big|_{x=0}=\alpha+\frac{\alpha}{1-\lambda}-2,\\
&\min_x\frac{\partial r_\alpha(x)}{\partial x}= -\alpha e^{\frac{2(1-\alpha)}{\alpha}}+\frac{\alpha}{1-\lambda}.
\end{align*}

Since $\min_x\frac{\partial r_\alpha(x)}{\partial x}$ is a monotone increasing function of $\alpha$ and it is positive for $\alpha=1$, there exists $\bar{\alpha}<1$, such that $\min_x \partial r_{\bar{\alpha}}(x)/\partial x \geq0$. 
Hence, $r_{\bar{\alpha}}(x)\geq0$ for $\alpha x<1$. 
 This concludes the results. 
\end{proof}

\begin{lemma}\label{lemma0}
Let $Q, \tilde{Q}\in\mathbb{R}^n$, then there exists an $0<\alpha<1$, such that 
\begin{align}
||Exp_\beta(Q)-Exp_\beta(\tilde{Q})||_1\leq 2\beta\alpha||Q-\tilde{Q}||_\infty,
\end{align}
where $Exp_\beta([q_1,...,q_n]):=[e^{\beta q_1},...,e^{\beta q_n}]/(\sum_i e^{\beta q_i})$.
\end{lemma}

\begin{proof}
Suppose that $\tilde{Q}=Q+X$, 
\begin{align*}
&||Exp_\beta(Q)-Exp_\beta(\tilde{Q})||_1=\sum_i\left|\frac{e^{\beta q_i}}{\sum_k e^{\beta q_k}}-\frac{e^{\beta q_i+\beta x_i}}{\sum_k e^{\beta q_k+\beta x_k}}\right|=\sum_i \frac{e^{\beta q_i}}{\sum_k e^{\beta q_k}}\left|1-\frac{e^{\beta x_i}}{\sum_k \frac{e^{\beta q_k}}{\sum_j e^{\beta q_j}}e^{\beta x_k}}\right|\\
&=\sum_i \lambda_i\left|1-\frac{e^{\beta x_i}}{\sum_k \lambda_k e^{\beta x_k}}\right|\leq \frac{1}{{\sum_k \lambda_j e^{\beta x_j}}}\sum_{i,k} \lambda_i\lambda_k\left|e^{\beta x_i}-e^{\beta x_k}\right|,
\end{align*}
where $\lambda_i:=\frac{e^{\beta q_i}}{\sum_k e^{\beta q_k}}$. Applying Lemma \ref{lemma2} will imply the result.
\end{proof}

\begin{lemma}\label{lemma4}
For two probability measures $u$ and $v$ defined over a countable space, we have
\begin{align*}
\sup_{||f||\leq c}\Big|\mathbb{E}_u[f]-\mathbb{E}_v[f]\Big|\leq c||u-v||_1.
\end{align*}
\end{lemma}

\subsection{Proof of Theorem \ref{thmma}}\label{thmmapr}
To conclude the result, first, we show that for any fixed $Q_{i}$ there exists $0\leq \alpha<1$ independent of $Q_{i}$, such that 
\begin{align}\label{eq:in10}
&|T_i(Q_{-i},Q_i)-T_i(\tilde{Q}_{-i},Q_i)|\leq \frac{B_i\gamma}{1-\gamma}2\alpha\beta(M-1)\max_{j\in-i}||Q_{j}-\tilde{Q}_{j}||.
\end{align}
Second, we show that for any fixed $Q_{-i}$,
\begin{align}\label{eq:in2}
&|T_i(Q_{-i},\tilde{Q}_i)-T_i(Q_{-i},\tilde{Q}_i)|\leq (\gamma+\frac{B_i\gamma}{1-\gamma}2\alpha\beta)||Q_{i}-\tilde{Q}_{i}||.
\end{align}
To show \eqref{eq:in10}, we start with the definition of $T_i$ that is 
\begin{align}\label{er01}
&T_i(Q_{-i},Q_i):= R_i\!+\!\gamma\mathbb{E}_{Exp_\beta\{Q_{-i}\}}\big[V_i\big],
\end{align}
where $V_i=\mathbb{E}_{Exp_\beta\{Q_i\}}[Q_i]$ and the expectation in \eqref{er01} is taken with respect to the policies of players $-i$. The left hand side of \eqref{eq:in10} can be written as follows,
\begin{align*}
|T_i(Q_{-i},Q_i)-T_i(\tilde{Q}_{-i},Q_i)|=\gamma\Big|\mathbb{E}_{Exp_\beta\{Q_{-i}\}}[V_i]-\mathbb{E}_{Exp_\beta\{\tilde{Q}_{-i}\}}[V_i]\Big|.
\end{align*}
Notice that if $||R_i||\leq B_i$ for some $B_i$, then $||V_i||\leq B_i(1+\gamma+\cdots)=B_i/(1-\gamma)$. 
Using the result of Lemma \ref{lemma4}, the right hand side of the above equation can be bounded as follows 
\begin{align}\label{F1}
&\gamma\Big|\mathbb{E}_{Exp_\beta\{Q_{-i}\}}[V_i]-\mathbb{E}_{Exp_\beta\{\tilde{Q}_{-i}\}}[V_i]\Big|\leq \frac{B_i\gamma}{1-\gamma}||Exp_\beta\{Q_{-i}\}-Exp_\beta\{\tilde{Q}_{-i}\}||_1,
\end{align}
 Given the result of Lemma \ref{lemma0}, there exists an $0<\alpha<1$ such that the right hand side of the above inequality can be bounded by
\begin{align*}
\frac{B_i\gamma}{1-\gamma}2\alpha\beta||Q_{-i}-\tilde{Q}_{-i}||.
\end{align*}
Note that $Q_{-i}=\bigoplus_{j\in-i}Q_j$, where $\vec{u}\bigoplus\vec{v}=[u_n+v_m]$ for $1\leq n\leq |\vec{u}|$ and $1\leq m\leq |\vec{v}|$.
Therefore, 
$$
||Q_{-i}-\tilde{Q}_{-i}||\leq (M-1)\max_{j\in-i}||Q_j-\tilde{Q}_j||.
$$
This concludes the inequality in \eqref{eq:in10}.

In order to prove \eqref{eq:in2}, we have
\begin{align*}
&|T_i(Q_{-i},Q_i)-T_i(Q_{-i},\tilde{Q}_i)|=\gamma\Big|\mathbb{E}_{Exp_\beta(Q_i)}\mathbb{E}_{Exp_\beta\{Q_{-i}\}}[Q_i]-\mathbb{E}_{Exp_\beta(\tilde{Q}_i)}\mathbb{E}_{Exp_\beta\{Q_{-i}\}}[\tilde{Q}_i]\Big|\\
&\leq \gamma\Big|\mathbb{E}_{Exp_\beta(Q_i)}\mathbb{E}_{Exp_\beta\{Q_{-i}\}}[Q_i]-\mathbb{E}_{Exp_\beta(Q_i)}\mathbb{E}_{Exp_\beta\{Q_{-i}\}}[\tilde{Q}_i]\Big|+\gamma\Big|\mathbb{E}_{Exp_\beta(Q_i)}[\phi]-\mathbb{E}_{Exp_\beta(\tilde{Q}_i)}[\phi]\Big|\\
&\leq \gamma||Q_i-\tilde{Q}_i||+\frac{B_i\gamma}{1-\gamma}2\alpha\beta||Q_i-\tilde{Q}_i||,
\end{align*}
where $\phi:= \mathbb{E}_{Exp_\beta\{Q_{-i}\}}[\tilde{Q}_i]$. The last inequality is obtained similar to the proof of \eqref{eq:in10} and the fact that $\mathbb{E}[g]\leq||g||$.

Final step is to combine \eqref{eq:in10} and \eqref{eq:in2}. 
\begin{align*}
&|T_i(Q_{-i},{Q}_i)-T_i(\tilde{Q}_{-i},\tilde{Q}_i)|\leq |T_i(Q_{-i},{Q}_i)-T_i({Q}_{-i},\tilde{Q}_i)|+|T_i(Q_{-i},\tilde{Q}_i)-T_i(\tilde{Q}_{-i},\tilde{Q}_i)|\\
&\leq (\gamma+\frac{B_i\gamma}{1-\gamma}2\alpha\beta M)||Q-\tilde{Q}||.
\end{align*}
In the above equation $||Q-\tilde{Q}||=\max_k ||Q_k-\tilde{Q}_k||$. This shows that $T_i$ is contraction if $B_i<\frac{(1-\gamma)^2}{2M\gamma\alpha\beta}$ in the infinite-time horizon problem.

\subsection{Proof of Theorem \ref{thmm2}}\label{thmm2pr}
We prove the result by showing that, if $||R_{i,F}||, ||R_i||$ are bounded, there exists $0\leq \alpha<1$, such that
\begin{align*}
&|U_i(Q^t_{-i},V^{t+1}_i)-U_i(\tilde{Q}^t_{-i},V^{t+1}_i)|\leq C_i(T+1)2\alpha\beta(M-1)\max_{j\in-i}||Q_{j}-\tilde{Q}_{j}||_\infty.
\end{align*}
This inequality is analogous to \eqref{eq:in10} and can be shown similarly by applying Lemma \ref{lemma0} and Lemma \ref{lemma4} and keeping in mind that if $||R_i||\leq C_i$ for some $C_i$, then $||Q_i||\leq C_i(T+1)$.

\subsection{Proof of Theorem 3}
We prove this statement by showing that there exists an $\alpha\in [0,1)$ such that for any agent $i$ and any $t$,  
\begin{align}\label{eq:main}
&  ||Q^t_{i,k+1}- Q^t_{i,k}||_\infty \leq \alpha ||Q^t_{i,k}- Q^t_{i,k-1}||_\infty.
\end{align}
where $||Q- \tilde{Q}||_\infty:=\sup_{x}\sup_{a} |Q(x,a)-\tilde{Q}(x,a)|$.
As a reminder, we have
\begin{align}\label{sys1}
&Q^t_{i,k}=B_i(O^{t}_{-i,k-1},Q^{t+1}_{i,k})=R_i+\Psi(O^t_{-i, k-1})\!+ \mathbb{E}_{Exp(Q_{i,k}^{t+1})}[Q_{i,k}^{t+1}],\\ \label{sys2}
&O_{i,k}^t=G(Q_{i,k}^{t-1},O_{i,k}^{t-1})=\mathbb{E}_{Exp(Q_{j,k}^{t-1})}[O_{j,k}^{t-1}],
\end{align}
Inequality \eqref{eq:main} and the result of Lemma \ref{lemma1} conclude the result. 
To show \eqref{eq:main}, let
\begin{align*}
& \Delta_k^t:=||Q_k^t-Q_{k-1}^t||_\infty,\\ \notag
&\Lambda_k^t:=||O_k^t-O_{k-1}^t||_1,
\end{align*}
in which for simplicity, the index corresponding to the agent is dropped. 
Using the update equations in Algorithm 3, i.e., Equation \extref{backw}, and Lemma \ref{lemma:soft_lip}, we obtain
\begin{align}\label{eq:up1}
& \Delta_k^t \leq L \Lambda_{k-1}^t + S\Delta_k^{t+1},\\
&\Delta_k^T \leq L \Lambda_{k-1}^T,
\end{align}
where $S:=(1+\xi\beta)$ is given in Lemma \ref{lemma:soft_lip}.

On the other hand, using Equation \extref{forw}, we have
\begin{small}
\begin{align*}
& \Lambda_k^t=||G(Q_{k}^{t-1},O_{k}^{t-1})-G(Q_{k-1}^{t-1},O_{k-1}^{t-1})||_1\\
&\leq||G(Q_{k}^{t-1},O_{k}^{t-1})-G(Q_{k-1}^{t-1},O_{k}^{t-1})||_1+||G(Q_{k-1}^{t-1},O_{k}^{t-1})-G(Q_{k-1}^{t-1},O_{k-1}^{t-1})||_1.
\end{align*}
\end{small}
Given the results of Lemma \ref{lemma:in20}, we have
\begin{align}
&\Lambda_k^t\leq 2\beta\Delta_k^{t-1}+\Lambda_k^{t-1},\\ \label{eq:up2}
&\Lambda_k^0=0.
\end{align}
From the inequalities in \eqref{eq:up1}-\eqref{eq:up2}, we obtain
\begin{small}
\begin{align*}
&\Delta_k^t\leq L\left(\sum_{j=0}^{T-t}S^{j}\Lambda_{k-1}^{t+j}\right),\\
&\Lambda_{k-1}^{t}\leq 2\beta\left(\sum_{i=0}^{t-1}\Delta_{k-1}^{t-i-1}\right).
\end{align*}
\end{small}
Combining the above inequalities implies
\begin{small}
\begin{align*}
&\Delta_k^t\leq 2L\beta\sum_{j=0}^{T-t}S^{j}\sum_{i=0}^{t+j-1}\Delta_{k-1}^{t+j-i-1}\leq 2L\beta\Delta_{k-1}\sum_{j=0}^{T-t}\sum_{i=0}^{t+j-1}S^{j}\overset{(a)}{\leq}  2L\beta\Delta_{k-1}\sum_{j=1}^{T}\sum_{i=0}^{j-1}S^{j}\\
&=2L\beta\Delta_{k-1}S\frac{TS^{T+1}-(T+1)S^T+1}{(S-1)^2}\leq 2L\beta\Delta_{k-1}S\frac{S^{T}T}{(S-1)}\leq \Delta_{k-1}\frac{2LT}{\xi}e^{\beta(T+1)\xi}.
\end{align*}
\end{small}
where $\Delta_{k-1}:=\max_{0\leq\tau\leq T}\Delta_{k-1}^{\tau}$. Inequality $(a)$ is due to the fact that $S>1$.
To obtain convergence, we require the coefficient of $\Delta_{k-1}$ to be less than one. Therefore, we obtain
\begin{align*}
&2LT\leq \xi e^{-\beta(T+1)\xi}.
\end{align*}

\begin{lemma}\label{lemma:in20}
We have
\begin{small}
\begin{align*}
&||G(Q_{k}^{t-1},O_{k}^{t-1})-G(Q_{k-1}^{t-1},O_{k}^{t-1})||_1\leq 2\beta\Delta_k^{t-1},\\
&||G(Q_{k-1}^{t-1},O_{k}^{t-1})-G(Q_{k-1}^{t-1},O_{k-1}^{t-1})||_1\leq \Lambda_k^{t-1}.
\end{align*}
\end{small}
\end{lemma}
\begin{proof}
Using the definition of transformation $U$, the left hand side can be written as follows
\begin{small}
\begin{align}\label{eq:in1}
&\int_x\left|\int_{x'}\sum_a \left(\pi_k^{t-1}(a|x')-\pi_{k-1}^{t-1}(a|x')\right)P(x|x',a)O_k^{t-1}(x')dx'\right|dx \\ \notag
&\leq\int_x\int_{x'} \sum_a\left|\pi_k^{t-1}(a|x')-\pi_{k-1}^{t-1}(a|x')\right|P(x|x',a)O_k^{t-1}(x')dx'dx \\ \notag
&=\int_{x'}\sum_a \left|\pi_k^{t-1}(a|x')-\pi_{k-1}^{t-1}(a|x')\right|O_k^{t-1}(x')dx'
\end{align}
\end{small}
where $\pi_k^{t-1}(a|x')\propto \exp(\beta Q_k^{t-1}(x',a))$. 
Based on the result of Lemma \ref{lemma0}, we can bound the above difference as follows,
\begin{small}
\begin{align*}
&\int_{x'}2\beta||Q_k^{t-1}(x',\cdot)-Q_{k-1}^{t-1}(x',\cdot)||_\infty O_k^{t-1}(x')dx'\leq2\beta||Q_k^{t-1}-Q_{k-1}^{t-1}||_\infty=2\beta\Delta_{k}^{t-1}.
\end{align*}
\end{small}
The left hand side of the second inequality is
\begin{small}
\begin{align*}
&\int_x\left|\int_{x'}\sum_a \pi_{k-1}^{t-1}(a|x')P(x|x',a)\left(O_k^{t-1}(x')-O_{k-1}^{t-1}(x')\right)dx'\right|dx\\
&\leq \int_x\int_{x'}\sum_a \pi_{k-1}^{t-1}(a|x')P(x|x',a)\left|O_k^{t-1}(x')-O_{k-1}^{t-1}(x')\right|dx'dx\\
&= \int_{x'}\left|O_k^{t-1}(x')-O_{k-1}^{t-1}(x')\right|dx'= \Lambda_k^{t-1}.
\end{align*}
\end{small}
\end{proof}

\begin{lemma}\label{lemma:in1}
For a given vectors $\alpha=(\alpha_1,...,\alpha_n)$, let
\begin{align*}
f_\alpha(Q):=\sum_j\frac{ \exp(\beta q_j)\alpha_j}{\sum_{i} \exp(\beta q_i)},
\end{align*}
where $Q=(q_1,...,q_n)$. 
Then, for any two arbitrary vectors $Q$ and $\widetilde{Q}$, we have
\begin{align}
|f_\alpha(Q)-f_\alpha(\widetilde{Q})|\leq \frac{\beta||\alpha||_d}{2}||Q-\widetilde{Q}||_\infty.
\end{align}
where $||\alpha||_d:=\max_{i,j}|\alpha_i-\alpha_j|$.
\end{lemma}
\begin{proof}
The norm of the gradient of $f_\alpha(\cdot)$ is given by 
\begin{align*}
||\nabla f_\alpha ||_1=\sum_j\frac{|\sum_{i} \beta(\alpha_j-\alpha_i)\exp(\beta(q_i+q_j))|}{(\sum_{i} \exp(\beta q_i))^2}.
\end{align*}
The right hand side of the above inequality can be written as
\begin{align*}
\sum_j b_j\frac{\sum_{i} \beta(\alpha_j-\alpha_i)\exp(\beta(q_i+q_j))}{(\sum_{i} \exp(\beta q_i))^2},
\end{align*}
where $b_j\in\{-1,1\}$ specifies the sign of the term within the absolute-value. Without loss of generality, we assume that $\alpha_1\leq...\leq \alpha_n$. In this case, $b_n=1$ and $b_1=-1$ but the rest of $b_j$s cannot be specified only based on $\alpha_j$s.
We define matrix $A\in\mathbb{R}^{n\times n}$, such that $A_{i,j}=(\alpha_i-\alpha_j)$, then 
\begin{small}
\begin{align}\label{eq:f2}
&||\nabla f_\alpha ||_1=\beta\frac{(b\circ\exp(\beta Q))^T A \exp(\beta Q)}{||\exp(\beta Q)||^2_1}\leq \beta\max_{b'\in\{0,1\}^n}\frac{(b'\circ\exp(\beta Q))^T A \exp(\beta Q)}{||\exp(\beta Q)||^2_1},
\end{align}
\end{small}
where $b\circ\exp(\beta Q)=(b_1\exp(\beta q_1),...,b_n\exp(\beta q_n))$. Let us define
\begin{align*}
&\varrho_{b'} := (b'\circ\exp(\beta Q))^T A \exp(\beta Q)=\exp(\beta Q)^T A_{b'} \exp(\beta Q),
\end{align*}
where $A_{b'}$ is matrix $A$ in which $i$th row is multiplied by $b'_i$.
In this case, 
$$
2\varrho_{b'}=\exp(\beta Q)^T (A_{b'}+A^T_{b'}) \exp(\beta Q).
$$
Note that $A+A^T=0$, and hence the absolute value of the $(i,j)$th entry of $A_{b'}+A^T_{b'}$ is zero iff $b_i' b_j'=1$, otherwise it is at most $2|\alpha_i-\alpha_j|$. Therefore, $A_{b'}+A^T_{b'}$ has most non-zero entries iff exactly half of $b'$ is +1 and the other half is -1. Let $I_1$ and $I_2$ denote the indices that $b'_i$ are +1 and -1, respectively and $|I_1|=\lfloor n/2\rfloor$, $|I_2|=\lceil n/2\rceil$ . 
Therefore, we obtain
\begin{small}
\begin{align}\notag
&2\varrho_{b'}\leq 2\sum_{i\in I_1}\sum_{j\in I_2}2|\alpha_i-\alpha_j|\exp(\beta(q_i+q_j))\leq 4||\alpha||_d\sum_{i\in I_1}\sum_{j\in I_2}\exp(\beta(q_i+q_j))\\ \label{eq:f3}
&= 4||\alpha||_d\left(\sum_{i\in I_1}\exp(\beta q_i)\right)\left(\sum_{j\in I_2}\exp(\beta q_j)\right).
\end{align}
\end{small}
Using the definition of $\varrho_{b'}$, Equations \eqref{eq:f2} and \eqref{eq:f3}, we have
\begin{small}
\begin{align*}
&||\nabla f_\alpha ||_1\leq \beta\max_{b'}\varrho_{b'}\leq \beta||\alpha||_d\max_{I_1, I_2}\frac{2\left(\sum_{i\in I_1}\exp(\beta q_i)\right) \left(\sum_{j\in I_2}\exp(\beta q_j)\right)}{\left(\sum_{i\in I_1}\exp(\beta q_i)+\sum_{j\in I_2}\exp(\beta q_j)\right)^2}
\leq \beta\frac{||\alpha||_d}{2}.
\end{align*}
\end{small}
The rest follows by the mean-value theorem, the fact that $f_\alpha$ is a continuous function of $Q$, and H\"older inequality.
\end{proof}

\begin{lemma}\label{lemma1}
For a given sequence $\{a_n\}_{n=1}^{\infty}$, if there exists $\alpha\in[0,1)$ such that 
\begin{align}\label{eq:sq}
|a_n-a_{n-1}|\leq\alpha|a_{n-1}-a_{n-2}|, \ \ \forall n\geq1,
\end{align}
then $a_n$ converges. 
\end{lemma}

\begin{lemma}\label{lemma:soft_lip}
Function $g(Q):=\mathbb{E}_{Exp(Q)}[Q]$ is Lipschitz with the constant
$S:=(1+\xi\beta)$, where $||Q||_\infty\leq \xi$, and $\beta$ is the inverse temperature.
\end{lemma}
\begin{proof}
This is a consequence of Lemma \ref{lemma:in1} and the fact that $\max_{i,j}|\alpha_i-\alpha_j|\leq 2||\alpha||_\infty$. Note that 
\begin{small}
\begin{align*}
&|g(Q)-g(\widetilde{Q})|\leq \Big|\mathbb{E}_{Exp(Q)}[Q]-\mathbb{E}_{Exp(\widetilde{Q})}[Q]\Big|+\Big|\mathbb{E}_{Exp(\widetilde{Q})}[Q]-\mathbb{E}_{Exp(\tilde{Q})}[\widetilde{Q}]\Big|\\
&\leq (\beta||Q||_\infty + 1)||Q-\widetilde{Q}||_\infty,
\end{align*}
\end{small}
where $||Q||_\infty:=\sup_{x}\sup_{a} |Q(x,a)|$.
\end{proof}

\subsection{Proof of Theorem \ref{thm:G}}\label{thm:G:p}
Differentiating the Lagrangian of the problem in \extref{irl_max},
\begin{align*}
&\Gamma= H(\vec{a}||\vec{x})+\sum_{i,k}\sum_\tau\theta_{k,i}\left(\mathbb{E}[F_{i,k}(\vec{x}(\tau),a_i(\tau))]-\tilde{\mathbb{E}}[F_{i,k}(\vec{x}(\tau),a_i(\tau))]\right)+\\
&\sum_{i,\tau}\sum C^{(i)}_{a_i(\tau),\vec{a}^{\tau-1},\vec{x}^\tau}\ \pi(a_i(\tau)|\vec{a}^{\tau-1},\vec{x}^\tau)\!+\!\sum_{i,\tau}\sum D^{(i)}_{\vec{a}^{\tau-1},\vec{x}^\tau}\left(\!\sum_{a_i(\tau)} \pi(a_i(\tau)|\vec{a}^{\tau-1},\vec{x}^\tau)\!-\!1\!\right),
\end{align*}
implies
\begin{align*}
&\frac{\partial\Gamma}{\partial \pi(a_i(t)|\vec{a}^{t-1},\vec{x}^t)}=C^{(i)}_{a_i(t),\vec{a}^{t-1},\vec{x}^t}+D^{(i)}_{\vec{a}^{t-1},\vec{x}^t}-\sum_{\vec{a}_{-i}(t)}\frac{ P(\vec{a}^{t},\vec{x}^{t})}{\pi(a_i(t)|\vec{a}^{t-1},\vec{x}^t)}-\\
&\sum_{\tau\geq t}\sum \frac{\partial P(\vec{a}^{\tau},\vec{x}^{\tau})}{\partial\pi(a_i(t)|\vec{a}^{t-1},\vec{x}^t)}\log \pi(\vec{a}(\tau)|\vec{a}^{\tau-1},\vec{x}^\tau)+\sum_{j,k}\theta_{k,j}\sum_{\tau\geq t}\sum\frac{\partial P(\vec{a}^{\tau},\vec{x}^{\tau})}{\partial\pi(a_i(t)|\vec{a}^{t-1},\vec{x}^t)}F_{j,k}(\vec{x}(\tau),a_j(\tau)),
\end{align*}
where for $\tau\geq t$,
\begin{align*}
\frac{\partial P(\vec{a}^{\tau},\vec{x}^{\tau})}{\partial\pi(a_i(t)|\vec{a}^{t-1},\vec{x}^t)}=\frac{ \pi(\vec{a}(t)|\vec{a}^{t-1},\vec{x}^{t})}{\pi(a_i(t)|\vec{a}^{t-1},\vec{x}^t)}P(\vec{a}^{t-1},\vec{x}^{t})P(\vec{a}_{t+1}^\tau,\vec{x}_{t+1}^\tau|\vec{a}^{t},\vec{x}^t).
\end{align*}
For simplicity, we denote $\pi(a_j(\tau)|\vec{a}^{\tau-1},\vec{x}^\tau)$ and $\pi(\vec{a}(\tau)|\vec{a}^{\tau-1},\vec{x}^\tau)$ by $\pi_j(\tau)$ and $\pi(\tau)$, respectively.
By equating the derivative of the Lagrangian to zero, we obtain
\begin{align*}
&C^{(i)}_{a_i(t),\vec{a}^{t-1},\vec{x}^t}+D^{(i)}_{\vec{a}^{t-1},\vec{x}^t}=P(\vec{a}^{t-1},\vec{x}^{t})\Bigg(1+\sum_{\vec{a}_{-i}(t)}\pi_{-i}(t)\log \pi_{-i}(t)+\log \pi_{i}(t)+\\
&\sum_{\tau> t}\sum_{\vec{a}_{-i}(t)} \pi_{-i}(t)\ \mathbb{E}\ [\log \pi(\tau)\big| \ \vec{a}^{t},\vec{x}^t]-\sum_{j,k}\theta_{k,j}\sum_{\tau\geq t}\sum_{\vec{a}_{-i}(t)} \pi_{-i}(t)\ \mathbb{E}\ [F_{j,k}(\vec{x}(\tau),a_j(\tau)) \ \big| \ \vec{a}^{t},\vec{x}^t]\Bigg),
\end{align*}
where $\pi_{-i}(t)=\prod_{j\in-i}\pi_j(t)$. 
The above equation suggests the following form for agent's $i$ policy at time $T$,
\begin{align*}
\log\pi_i(T)=&\exp\Big(\theta_i^TF_i(\vec{x}(T),a_i(T))-\softmax_{a'}\theta_i^TF_i(\vec{x}(T),a')\Big)\ \ \forall\ i,\\
\log\pi_{i}(T-1)\propto&\exp\Big(\theta_i^TF_i(\vec{x}(T-1),a_i(T-1))+\!\!\!\!\!\!\sum_{\vec{a}_{-i}(T-1)}\!\!\!\!\!\! \pi_{-i}(T-1)\\
&P(\vec{x}_{-i}(T)|\vec{x}_{-i}(T-1),\vec{a}_{-i}(T-1))\big[\softmax_{a'}\theta_i^Tf_i(\vec{x}(T),a'_i(T))\big]\Big)\ \ \forall\ i,\\
\hspace{.3cm}\vdots&
\end{align*}
where $\theta_i=[\theta_{1,i},...,\theta_{K,i}]$ and $F_i=[F_{i,1},...,F_{i,K}]^T$.

\subsection{Proof of Theorem \ref{thm:gra}}

We have 
\begin{align}\label{gg1}\notag
&H(\vec{a}||\vec{x})=\sum_{t\leq T}\mathbb{E}_{\vec{x}^{t},\vec{a}^t}[-\log \pi^t(a(t)|\vec{x}(t))]=-\sum_{t\leq T}\sum_i\mathbb{E}_{\vec{x}^{t},\vec{a}^t}[\theta_i^TF_i(\vec{x}(t),a_i(t))]\\\notag
&-\sum_{t=1}^{T-1}\sum_i\mathbb{E}_{\vec{x}^{t+1},\vec{a}^{t}}[\softmax W^{t+1}_i(H^{t+1},a)]+\sum_{t=1}^T\sum_i\mathbb{E}_{\vec{x}^{t},\vec{a}^{t-1}}[\softmax W^t_i(\vec{x}(t),a)]\\
&=-\sum_{t\leq T}\sum_i\mathbb{E}_{\vec{x}^{t},\vec{a}^t}[\theta_i^TF_i(\vec{x}(t),a_i(t))]+\sum_i\mathbb{E}_{\vec{x}^{1}}[\softmax W^1_i(H^1,a)].
\end{align}
On the other hand, we have
\begin{align*}
\mathbb{E}_{\vec{x}(t)}[\nabla_{\theta_j}\softmax W^t_i(\vec{x}(t),a)]=\mathbb{E}[\sum_a\pi_i^t(a|\vec{x}(t))\nabla_{\theta_j}W^t_i(\vec{x}(t),a) ]=\mathbb{E}_{\vec{x}(t)\cup a_i(t)}[\nabla_{\theta_j}W^t_i(\vec{x}(t),a) ].
\end{align*}
Using the above equality, for any deterministic function $\Phi$, we have
\begin{align}\label{goo}
&\mathbb{E}_{\vec{x}(t)}\Big[\sum_a \Phi(\vec{x}(t)) \nabla_{\theta_j}\pi_{k}(a|\vec{x}(t))\Big]=\\
&\mathbb{E}_{\vec{x}(t)}\Big[\sum_a \Phi(\vec{x}(t)) \pi_{k}(a|\vec{x}(t))\big(\nabla_{\theta_j}W_k^t(\vec{x}(t),a)-\nabla_{\theta_j}\softmax_b W_k^t(\vec{x}(t),b)\big)\Big]=0.
\end{align}
Taking the derivative of the last term of (\ref{gg1}) implies 
\begin{align*}
&\nabla_{\theta_j} \mathbb{E}_{\vec{x}^{1}}[\softmax W^1_j(H^1,a)]=\mathbb{E}_{\vec{x}^1,a}[F_j(\vec{x}(1),a)]+\mathbb{E}_{\vec{x}^2,\vec{a}^2}[\nabla_{\theta_j}W_j^2(H^2,a_j(2))]+\\
&\mathbb{E}_{\vec{x}^1,a}\Big[\pi_{-j}(\vec{a}_{-j}(1)|\vec{x}(t))\softmax_c W_j^2(H^2,c)\sum_{k\in-j}\nabla_{\theta_j}\big(W_k^1(H^1,a_k(1))-\softmax_dW_k^1(H^1,d)\big)\Big].
\end{align*}
Due to (\ref{goo}), the above last term is  zero, and by pushing the recursion to the final step, we obtain
\begin{align*}
&\nabla_{\theta_j} \mathbb{E}_{\vec{x}^{1}}[\softmax W^1_j(H^1,a)]=\sum_{t\leq T}\mathbb{E}_{\vec{x}^t,\vec{a}^t}[F_j(\vec{x}(t),a_j(t))].
\end{align*}
Similarly, when $i\neq j$, one can obtain 
\begin{align*}
&\nabla_{\theta_j} \mathbb{E}_{\vec{x}^{1}}[\softmax W^1_i(H^1,a)]=\mathbb{E}_{\vec{x}^2,\vec{a}^2}[\nabla_{\theta_j}W^2_i(H^2,a_j(2))]\\
&=\cdots=\mathbb{E}_{\vec{x}^T,\vec{a}^T}[\nabla_{\theta_j}W^T_i(\vec{x}(t),a_j(T))]=0.
\end{align*}
Going back to the derivative of the dual function, we will have two terms: $H(\vec{a}||\vec{x})$ and $\Theta(\mathbb{E}[F]-\tilde{\mathbb{E}}[F])$. Due to the above calculation, the derivative of the first term is zero, with respect to $\theta_i$ and the derivative of the second term concludes the result.

\bibliography{ref}
\bibliographystyle{abbrvnat}

\end{document}